\documentclass{article}
\usepackage{times}
\usepackage[final]{common_style}
\usepackage{hyperref}
\usepackage{url}
\usepackage{amsfonts}
\usepackage{graphicx}
\usepackage{amsthm}
\usepackage{amsmath}
\usepackage{multirow}
\theoremstyle{definition}
\newtheorem{proposition}{Proposition}

\title{Fine-tuning of Language Models with Discriminator}

\author{Vadim Popov \& Mikhail Kudinov
\\
Samsung R\&D Institute Russia\\
Moscow, Russia \\
\texttt{v.popov@samsung.com,m.kudinov@samsung.com}
}

\begin{document}

\maketitle

\begin{abstract}
Cross-entropy loss is a common choice when it comes to multiclass classification tasks and language modeling in particular. Minimizing this loss results in language models of very good quality. We show that it is possible to fine-tune these models and make them perform even better if they are fine-tuned with sum of cross-entropy loss and reverse Kullback-Leibler divergence. The latter is estimated using discriminator network that we train in advance. During fine-tuning probabilities of rare words that are usually underestimated by language models become bigger. The novel approach that we propose allows us to reach state-of-the-art quality on Penn Treebank: perplexity decreases from $52.4$ to $52.1$. Our fine-tuning algorithm is rather fast, scales well to different architectures and datasets and requires almost no hyperparameter tuning: the only hyperparameter that needs to be tuned is learning rate.
\end{abstract}

\section{Introduction}\label{intro}

Normally, when it is necessary to fine-tune a model on the same dataset that was used for initial training of this model, various optimization tricks are applied. For instance, one may think about changing optimizer type (in papers by \cite{AWD_LSTM}, \cite{MoS} and \cite{DOC_LM} language models that reach state-of-the-art performance are fine-tuned using Averaged Stochastic Gradient Descent (\cite{ASGD})). It is also possible to use complex learning rate schedules: for example, to increase learning rate for several epochs (it can help to get out of local minima or saddle points) or to use cyclical learning rates (\cite{cyclical_lr}).

In our paper we describe another way that can help in fine-tuning models that are designed to solve multiclass classification tasks by minimizing cross-entropy loss. We show that it is possible to use new loss function (namely, sum of cross-entropy loss and reverse Kullback-Leibler divergence) and fine-tune the initial model to optimize this new loss. We prove that such fine-tuning allows to reduce probabilities of words that for some reason were overestimated during initial training and demonstrate that probabilities of underestimated words can increase with almost no negative effect on model performance on other words.

The described scheme needs careful estimation of reverse Kullback-Leibler divergence (KL-divergence). For this purpose we first train discriminator neural network that has the same architecure as the inital model. Its outputs help to estimate new loss that is used for fine-tuning. Although this approach was inspired by generative adversarial networks (GANs by \cite{GAN}), our algorithm has very little in common with this type of generative networks. We do not fine-tune the initial language model to ``fool'' the discriminator -- the discriminator is only used to estimate reverse KL-divergence. Loss functions for language model and discriminator are very different in contrast with GANs where generator and discriminator share the same value function and play minimax game. Also, our fine-tuning procedure is plain and stable. We need to train discriminator till optimality and then fine-tune the initial model till optimality, and it's sufficient to do this only once.

Our main contributions are: 1) we propose a novel approach of fine-tuning language models by adding a special term to standard cross-entropy loss function, 2) we achieve state-of-the-art result on a popular language modeling benchmark Penn Treebank (PTB, \cite{PTB}), 3) we show that our fine-tuning algorithm is relatively fast, easy to train and applicable to large-scale datasets, 4) we expalin theory that lies behind the optimization scheme that we propose.

\section{Fine-tuning of language models}\label{fine_tune_general}

Standard language modeling task is to predict the next word $w_N$ given its left context $c = \{w_1\dots w_{N-1}\}$. One of the most effective ways to do that is to train a recurrent neural network, usually LSTM (\cite{LSTM}) or GRU (\cite{GRU}), to minimize cross-entropy loss between real data distribution $p$ and distribution generated by the neural network $q$:
\begin{equation}\label{eq:ce_loss}
CE(p||q) = -\sum_{w \in W}{p(w)\log q(w)},
\end{equation}
where $W$ is the vocabulary of all possible words language model can predict.
Cross-entropy is widely used because this loss function is the expectation under probability measure $p$ of negative log-likelihood of $q$. So, minimizing it makes sense from the point of view of statistics and, which is also very important, cross-entropy is easy to estimate -- we can just sample from $p$ and evaluate negative log-likelihood of $q$ on these samples.

Minimizing cross-entropy is equivalent to minimizing KL-divergence because they differ by a constant term (by entropy of real data distribution $p$):
\begin{equation}\label{eq:KL_div}
KL(p||q) = \sum_{w \in W}{p(w)\log \frac{p(w)}{q(w)}}
\end{equation}

Cross-entropy and KL-divergence are known to be asymetric -- since they both are expectations of some expressions under measure $p$, they pay less attention to the regions where $p$ is very small. For language modeling it means that probabilities under measure $q$ of very rare words (and also of some frequent words that appear very rarely in certain contexts) may be incorrect.

This drawback can be handled with if we consider reverse KL-divergence, i.e. KL-divergence between neural network distribution $q$ and real data distribution  $p$:
\begin{equation}\label{eq:KL_rev}
KL(q||p) = \sum_{w \in W}{q(w)\log \frac{q(w)}{p(w)}}
\end{equation}

Along with correcting the mentioned above issue we still want good statistical properties of the language model (i.e. we want likelihood of $q$ to be high). This leads to the following scheme: 1) we train a language model to minimize $CE(p||q)$, 2) then we fine-tune the language model to minimize $CE(p||q) + KL(q||p)$. 

However, in order to implement the second step we need to estimate reverse KL-divergence. The problem is that it is expectation under $q$ of some expression depending on $p$ which is intractable -- the only thing we can do with $p$ is to sample from it. 

This discussion resembles the one about GANs (e.g. by \cite{WGAN}). So, we decided to borrow the idea of training a discriminator from the paper where GANs were first introduced (\cite{GAN}).

\subsection{Training discriminator}\label{discriminator_general}

The key idea is that if we have a well-trained discriminator, we can then estimate $p$ in reverse KL-divergence since we know (see Proposition 1 in \cite{GAN}) that the optimal discriminator output distribution is a rational expression of $p$ and $q$. The same idea has recently been used for generator samples selection in GANs (\cite{Discriminator}).

We consider language modeling task, so the distributions $p$ and $q$ that we spoke about previously are actually conditional probability distributions $p(\cdot|c)$ and $q(\cdot|c)$ -- probabilities of the next word given its left context $c$. So, we want our discriminator to output \textit{conditional} probability that the next word comes from neural network distribution $q$. Let us denote this distribution by $r_{\varphi}(\cdot |c)$ where $\varphi$ are the parameters of the discriminator. The architecture of discriminator is exaclty the same as the architecture of the language model, only the interpretation of logits change -- in the language model softmax function is applied to them to get the probability of the next word $q(\cdot |c)$ whereas in the discriminator we apply sigmoid function to the logits to get the probability $r_{\varphi}(\cdot |c)$ that the next word was generated by the language model rather than came from real distribution $p(\cdot|c)$.

Loss function for discriminator $D(c,\varphi)$ is the same as the one used in standard GANs and is quite naturally given by expressions
\begin{equation}\label{eq:disc_loss_general}
D(c,\varphi) = D^{(q)}(c,\varphi) + D^{(p)}(c,\varphi),
\end{equation}
\begin{equation}\label{eq:disc_loss_detail}
D^{(q)}(c,\varphi) = -\mathbb{E}_{w\sim q_{0}(\cdot |c)}[\log{r_{\varphi}(w|c)}], \ \ D^{(p)}(c,\varphi) = -\mathbb{E}_{w\sim p(\cdot |c)}[\log(1-{r_{\varphi}(w|c)})],
\end{equation}

where $q_{0}$ is the probability distribution given by the initial language model (well-trained to minimize cross-entropy loss). We put zero subscript here to underline that the parameters of this language model are fixed. We will stick to this notation for the rest of the paper.

In practice, for any real context $c$ and any true next word $w^{*}$ from the training corpus we will estimate this loss by the following expression:
\begin{equation}\label{eq:disc_loss_eval}
\hat{D}(c,w^{*},\varphi) = -\sum_{w\in W}{q_{0}(w|c)\log{r_{\varphi}(w|c)}}  -\log{(1-r_{\varphi}(w^{*}|c))}
\end{equation}
 
As for left contexts $c$, we train discriminator only on examples of real contexts, so the overall loss function $D(\varphi)$ that we actually minimize is given by $D(\varphi)=\mathbb{E}_{c\sim p(\cdot)}[D(c,\varphi)]$.

It is worth mentioning that discriminator has to decide whether the next word comes from $p$ or $q$ given not only its left context $c$ (which is clear from notation) but also given the knowledge that this context $c$ comes from real distribution $p$. It becomes clear when we think of how we train this discriminator -- we always observe only real contexts $c$ while training. We design discriminator training in this manner because our final goal is to fine-tune language model whereas discriminator plays an auxiliary role of delivering an approximation to real data distribution $p(\cdot |c)$, and we need to estimate this distributiuon only for real contexts $c$.

\subsection{Fine-tuning with discriminator}\label{finetuner_general}

If the discrimator was trained till optimality, the following equation holds (see proof in \cite{GAN}):
\begin{equation}\label{eq:disc_expression}
r(w|c)=\frac{q_{0}(w|c)}{q_{0}(w|c)+p(w|c)}
\end{equation}
We omit the subscript $\varphi$ here for simplicity. From this equation we can get an estimation for real data distribution:
\begin{equation}\label{eq:p_expression}
\hat{p}(w|c)=q_{0}(w|c)\frac{1-r(w|c)}{r(w|c)}
\end{equation}

Language model loss during fine-tuning is given by
\begin{equation}\label{eq:ftd_loss}
L(c, \theta) = CE(p(\cdot |c)||q_{\theta}(\cdot |c)) + KL(q_{\theta}(\cdot |c)||p(\cdot |c)),
\end{equation}

where $\theta$ are parameters of the fine-tuned model. Cross-entropy term is calculated as usual. As for reverse KL-divergence, we can calculate it using the following identity:
\begin{equation}\label{eq:rev_KL_identity}
KL(q_{\theta}(\cdot |c)||p(\cdot |c))=\mathbb{E}_{w\sim q_{\theta}(\cdot |c)}\left[\log{\frac{q_{\theta}(w|c)}{p(w|c)}}\right]=\mathbb{E}_{w\sim p(\cdot |c)}\left[\frac{q_{\theta}(w|c)}{p(w|c)}\log{\frac{q_{\theta}(w|c)}{p(w|c)}}\right]
\end{equation}
Plugging in the estimation for $p$ from (\ref{eq:p_expression}) into expectation with respect to $p$ from (\ref{eq:rev_KL_identity}), we obtain estimation of total loss for any real context $c$ and true next word $w^{*}$:
\begin{equation}\label{eq:ftd_loss_eval_t}
t(c,w,\theta) = \frac{q_{\theta}(w|c)}{q_{0}(w|c)}\cdot \frac{r(w|c)}{1-r(w|c)}
\end{equation}
\begin{equation}\label{eq:ftd_loss_eval}
\hat{L}(c,w^{*},\theta)=-\log{q_{\theta}(w^{*}|c)}+t(c,w^{*},\theta)\log{t(c,w^{*},\theta)}
\end{equation}

As in the case of discriminator training, the overall loss function $L(\theta)$ that we actually minimize during model fine-tuning is given by $L(\theta)=\mathbb{E}_{c\sim p(\cdot)}[L(c,\theta)]$.

Interestingly, there is an alternative way to think of how we get estimation (\ref{eq:p_expression}). This estimation is actually the result of training a language model with Noise Contrastive Estimation (NCE, \cite{NCE}) loss without normalization constraint. It's easy to see that if we take the initial model distribution $q_0$ as noise distribution, then training a language model with NCE loss is equivalent to training discriminator as described in section \ref{discriminator_general} and obtaining $\hat{p}$ from its logits by formula (\ref{eq:p_expression}).

\section{Analysis}\label{analysis}

\begin{figure}[h]
\begin{minipage}[h]{0.49\linewidth}
\center{\includegraphics[width=1.0\linewidth]{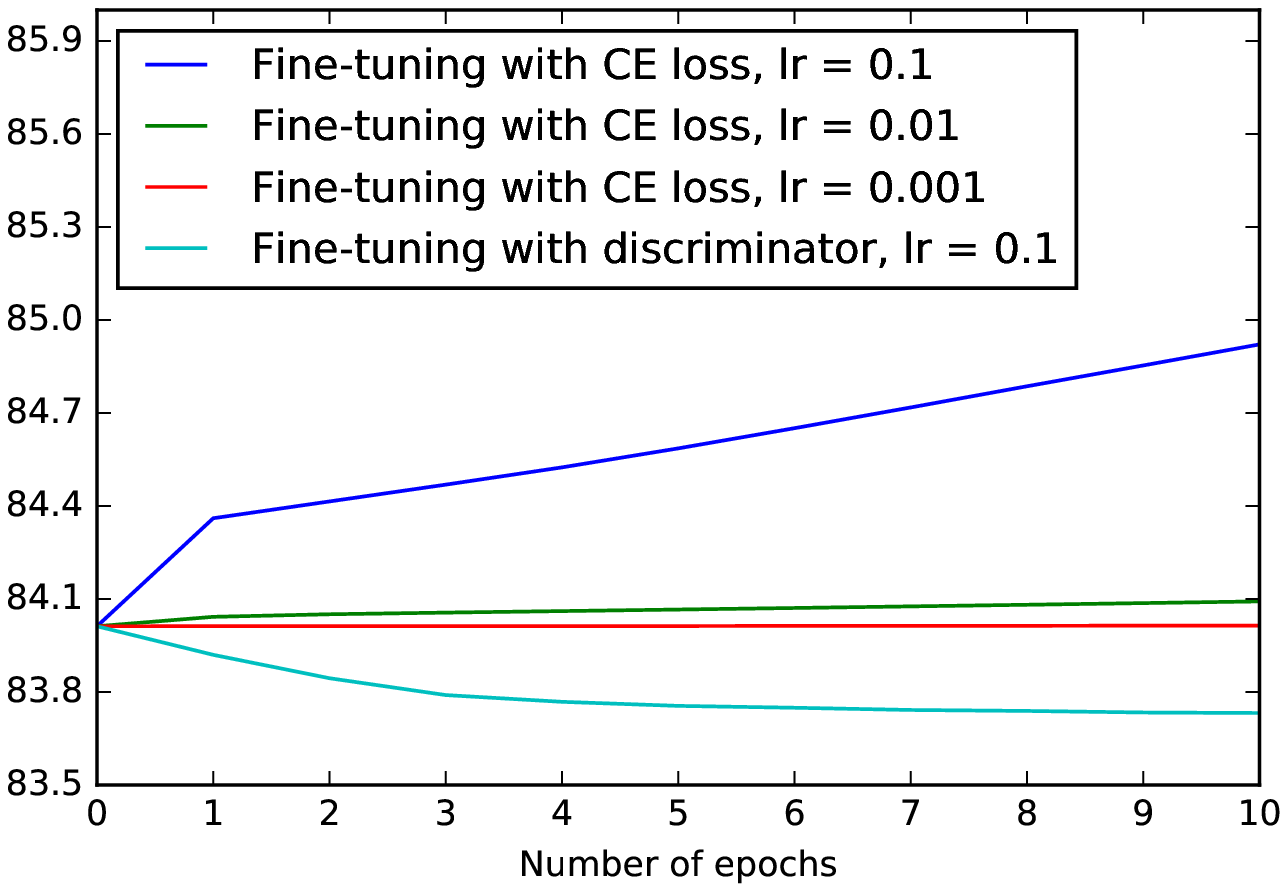} \\ }
\end{minipage}
\hfill
\begin{minipage}[h]{0.49\linewidth}
\center{\includegraphics[width=1.0\linewidth]{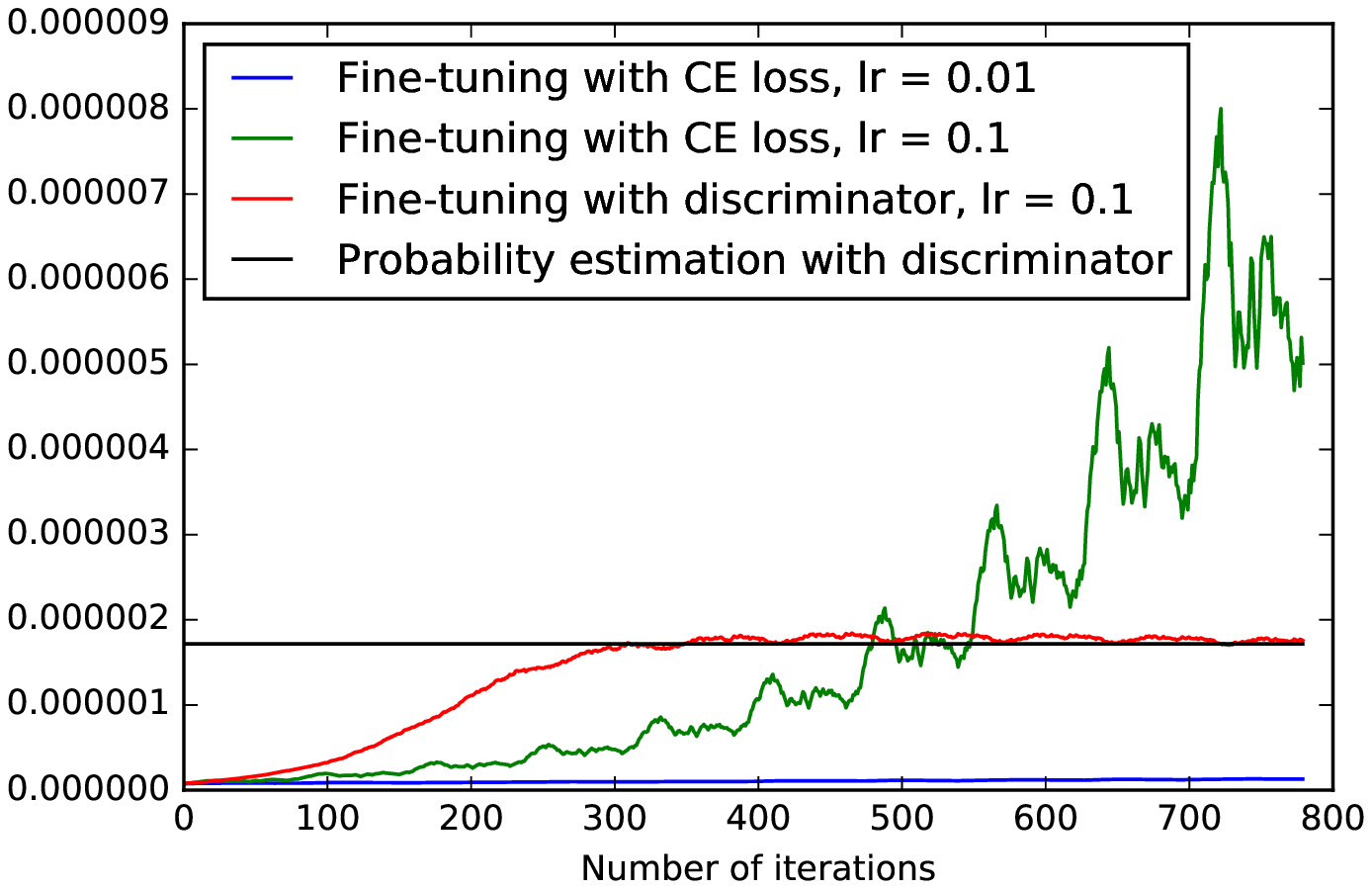} \\ }
\end{minipage}
\caption{Fine-tuning curves (synthetic data). The left plot shows validation perplexity. The right plot shows probability of the chosen word on each SGD iteration (given some fixed context).}
\label{pic:saddle}
\end{figure}

We analyze and illustrate basic properties of the proposed fine-tuning scheme on a synthetic dataset because in this case we know true probability distribution $p$ for sure. The structure of this dataset is much simpler than the one of natural language but it's sufficient for the purposes of this section.

We generate trigram probabilities according to a certain rule so that we have quite frequent words and quite rare words. Then we sample from this trigram distribution and get training, validation and test sets. Details of generating trigram probabilities, some dataset statistics, model architecture and also training and fine-tuning parameters are not of great importance; they are given in appendix \ref{artdataset}.

First, LSTM-based language model is trained till optimality. We'll refer to this model as the ``initial model''. Next, we choose a word and set the bias coefficient before the output softmax layer corresponding to this word to a very large negative value. Modifying this coefficient in the described manner makes the model assign extremely low probabilities to the chosen word in any context. We will refer to the model with modified coefficient as the ``perturbed model'' and we'll call the word corresponding to the modified coefficient the ``chosen word''.

Table \ref{tab:imbalance} shows the results of fine-tuning of the initial and perturbed models with our method. In the current experiment, by frequent we mean top $50$ most frequent words from training corpus; rare words are the remaining $950$ words. We will also say that a word is almost frequent if it is in the top $100$ most frequent words among rare words. The chosen word (the one whose bias we modified) belongs to almost frequent words. To measure imbalance between frequent and rare words we compute the ratio between perplexity calculated on the test set only for frequent words and only for rare words. We also calculate mean and standard deviation of differences between log-likelihoods of model predictions $q$ and true probabilities $p$.

Table \ref{tab:imbalance} gives evidence that generally language models tend to overestimate frequent words and underestimate rare words -- values in ``Ratio'' column for all language models are higher than the one for true distribution. Large values of standard deviation in the last column also underline that.

\begin{table}[t]
\caption{Fine-tuning results (synthetic data). Mark ``ftd.'' stands for the fine-tuned model. ``Ratio'' equals rare words perplexity divided by frequent words perplexity. The last column shows the difference between log-likelihoods of $q$ and $p$ (mean $\pm$ standard deviation calculated on the test set).}
\begin{center}
\begin{tabular}{|c|c|c|c|c|c|c|c|}
\hline
Estimating with &Test ppl. &Freq. ppl. &Rare ppl. &Ratio &$\log{q}-\log{p}$\\ \hline
True distribution $p$ &$78.14$ &$29.74$ &$766.53$ &$25.8$ &$0$\\
\hline
Initial model &$84.89$ &$31.27$ &$899.76$ &$28.8$ &$-0.083\pm 0.445$\\
\hline
Initial model ftd.&$84.88$ &$31.30$ &$897.30$ &$28.7$ &$-0.083\pm 0.445$\\
\hline
Perturbed model &$86.01$ &$31.24$ &$942.71$ &$30.2$ &$-0.096\pm 0.591$\\
\hline
Perturbed model ftd.&$85.63$ &$31.27$ &$926.78$ &$29.6$ &$-0.092\pm 0.521$\\
\hline
\end{tabular}
\end{center}
\label{tab:imbalance}
\end{table}

\begin{figure}[h]
\begin{minipage}[h]{0.49\linewidth}
\center{\includegraphics[width=1.0\linewidth]{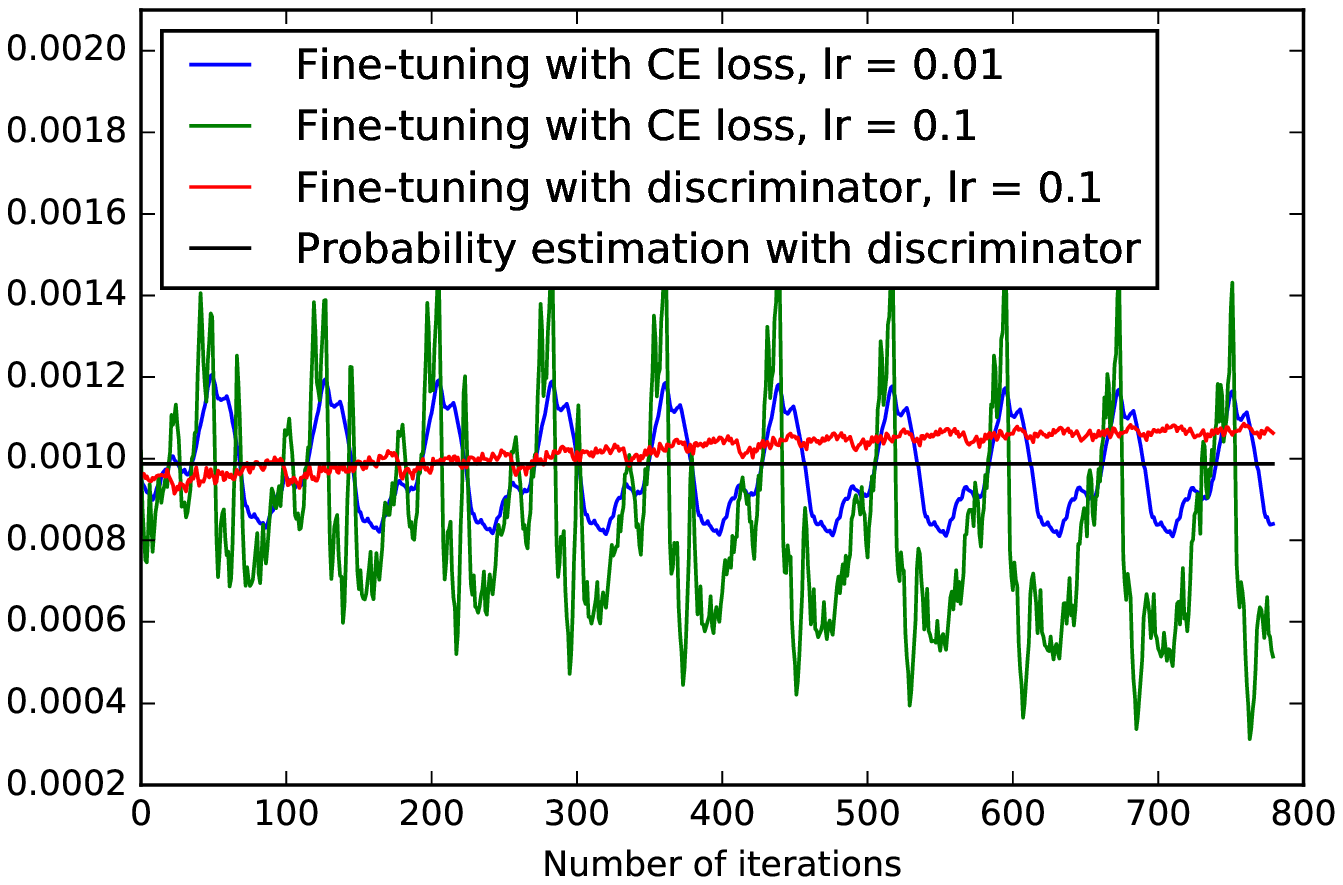} \\ }
\end{minipage}
\hfill
\begin{minipage}[h]{0.49\linewidth}
\center{\includegraphics[width=1.0\linewidth]{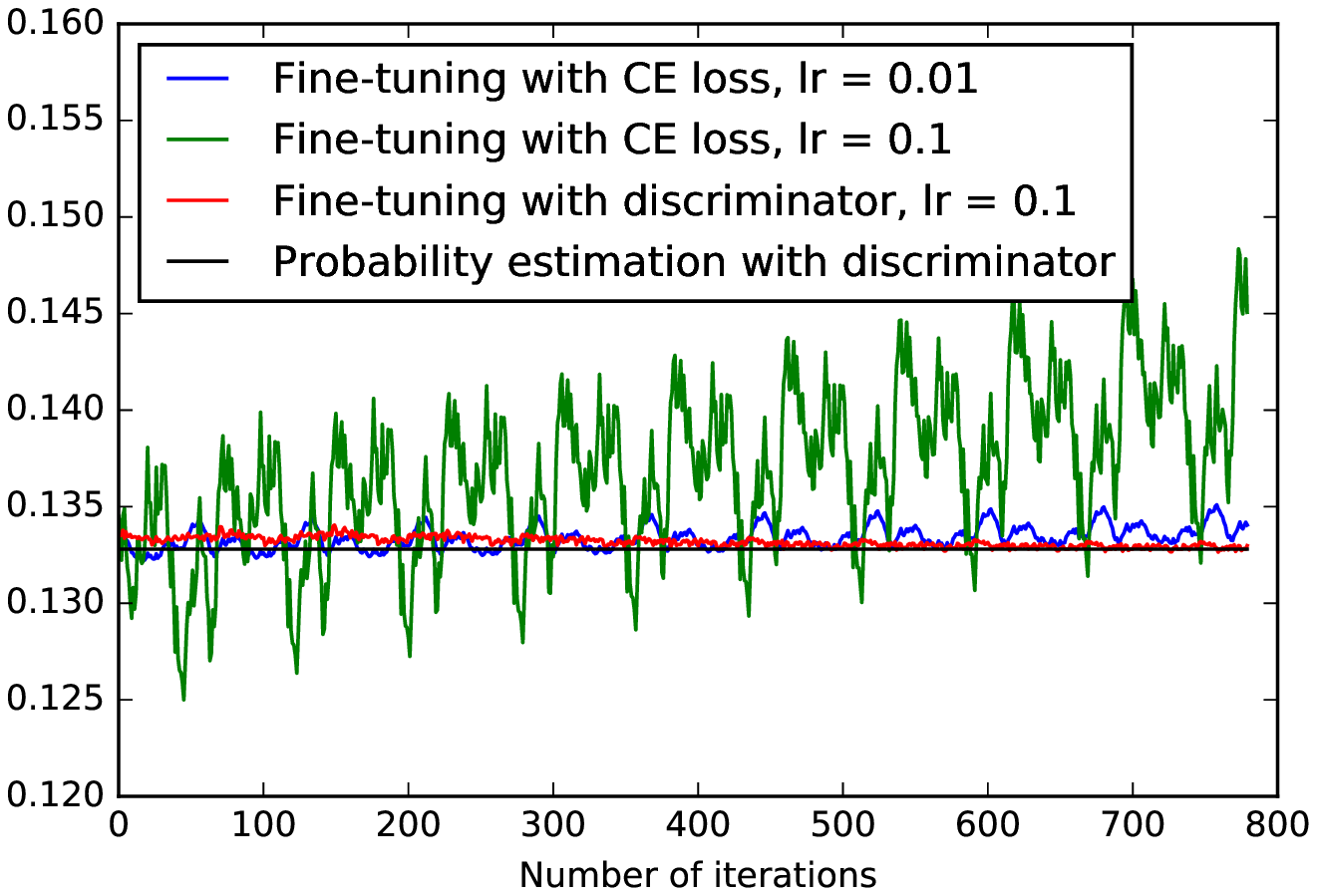} \\ }
\end{minipage}
\caption{Fine-tuning curves (synthetic data). The left and the right plot show probabilities of an almost frequent and a frequent word on each SGD iteration correspondignly (given some fixed contexts).}
\label{pic:freq_words}
\end{figure}

As for fine-tuning with discriminator, the initial model doesn't benefit significantly from it. But even in this case we can observe that fine-tuning acts exactly as we expected: it reduces the imbalance between frequent and rare words. However, this imbalance is still far from the one for real data distribution. 

Fine-tuning of the perturbed model works much better and is illustrated by Figure \ref{pic:saddle}. It's interesting to note that the perturbed model can't be fine-tuned with cross-entropy loss even though this model was obtained by modifying one of the parameters of the initial model rather than by some training process. To get an idea of why our fine-tuning algorithm works in this case we can look at how probabilities of different words change during different fine-tuning procedures (see Figures \ref{pic:saddle}, \ref{pic:freq_words}).

First of all, if we calculate perplexity given by the model fine-tuned with discriminator on all words except the chosen one, it doesn't change and stays at the level $84.56$. So, fine-tuning with discriminator performs well mostly due to its ability to increase probability of the chosen word. Figure~\ref{pic:saddle} demonstrates that standard fine-tuning can significantly increase this probability only for sufficiently large learning rates (of order $10^{-1}$) while small learning rates (of order $10^{-2}$) lead to almost the same probability as before fine-tuning. But the same figure shows that large learning rates get the model very far from the stationary point (most likely saddle point) it is already in. Moreover, Figure \ref{pic:freq_words} demonstrates that for large learning rates word probabilities tend to oscillate much which may also have negative impact on fine-tuning. So, during fine-tuning with cross-entropy loss we can't increase probabilitiy of the underestimated word and keep performing well and stable on all other words at the same time. On the contrary, fine-tuning with discriminator is able to increase probability of the chosen word approximately up to the level given by the estimation $\hat{p}$ avoiding harmful oscillation. The following proposition explains why it is so. 

\begin{proposition}\label{prop}
Let $\hat{p}=\hat{p}(w^{*}|c)$, \ $q_{\theta}=q_{\theta}(w^{*}|c)$, \ $\textbf{v}=\nabla_{\theta} q_{\theta}$, \ $\varepsilon>0$. Then 
\begin{equation}\label{eq:total_grad}
\nabla_{\theta}(-\log{q_{\theta}})=-\frac{1}{q_{\theta}}\cdot \textbf{v}, \ \ \ \nabla_{\theta}\left(\frac{q_{\theta}}{\hat{p}}\log{\frac{q_{\theta}}{\hat{p}}}\right)=\left(\frac{1}{\hat{p}}\left(1+\log{\frac{q_{\theta}}{\hat{p}}}\right)\right)\cdot \textbf{v}
\end{equation}
\begin{equation}\label{eq:ce_part_grad}
q_{\theta}=\hat{p}-\varepsilon \implies \nabla_{\theta}\left(-\log{q_{\theta}}+\frac{q_{\theta}}{\hat{p}}\log{\frac{q_{\theta}}{\hat{p}}}\right)=\left(-\frac{2\varepsilon}{\hat{p}^{2}}-\frac{3}{2}\frac{\varepsilon^{2}}{\hat{p}^{3}}+o(\varepsilon^{2})\right)\cdot \textbf{v}\end{equation}
\begin{equation}\label{eq:klrev_part_grad}
q_{\theta}=\hat{p}+\varepsilon \implies \nabla_{\theta}\left(-\log{q_{\theta}}+\frac{q_{\theta}}{\hat{p}}\log{\frac{q_{\theta}}{\hat{p}}}\right)=\left(\frac{2\varepsilon}{\hat{p}^{2}}-\frac{3}{2}\frac{\varepsilon^{2}}{\hat{p}^{3}}+o(\varepsilon^{2})\right)\cdot \textbf{v}
\end{equation}
\end{proposition}

\begin{proof}
Formulas (\ref{eq:total_grad}) are obvious. Applying Taylor's expansion proves formulas (\ref{eq:ce_part_grad}) and (\ref{eq:klrev_part_grad}). For details see appendix \ref{appendix}.
\end{proof}

Comparison of right-hand sides of formulas (\ref{eq:ce_part_grad}) and (\ref{eq:klrev_part_grad}) with (\ref{eq:total_grad}) reveals that the steps in our fine-tuning scheme are much smaller (by a factor of $\frac{\varepsilon}{\hat{p}}$) than if we were fine-tuning with just cross-entropy loss. It explains why for the same learning rates probabilities oscillate much more aggressively when we fine-tune with cross-entropy loss than when we perform fine-tuning with discriminator.

The proposition also implies that when we fine-tune with discriminator, we take a step in the direction $\textbf{v}$ towards increasing $q_{\theta}$ only when $q_{\theta}<\hat{p}$, and when $q_{\theta}>\hat{p}$ we take a step in the opposite direction $-\textbf{v}$ towards decreasing $q_{\theta}$ even though it is probability of the true next word. When we consider standard fine-tuning with cross-entropy loss, we always take a gradient step in the direction $\textbf{v}$ to increase probability $q_{\theta}$ of the true word. So, we increase true word probability even if it is already big, which may sometimes be unreasonable and lead to overestimation of frequent words.

To sum it up, if in the end of training language model approached a saddle point where imbalance between frequent and rare words prediction is too high and further fine-tuning with cross-entropy loss is impossible, fine-tuning with discriminator can help to escape from this point to the regions where the imbalance is lower.

The most vulnerable point in the proposed scheme is that we tend to make gradient steps in the direction of the estimation calculated using discriminator, and there is no guarantee that this estimation is accurate -- it depends on the discriminator's quality. For example, the perturbed model assigns probability of order $10^{-8}$ to the chosen word. Probability estimation obtained using discriminator is of order $10^{-6}$, but the true probability is of order $10^{-3}$. If we could train discriminator better, we would be able to achieve better results.

\section{Experiments}\label{experiments}

\subsection{Penn Treebank}\label{exp_ptb}

\begin{figure}[h]
\begin{minipage}[h]{0.49\linewidth}
\center{\includegraphics[width=1.0\linewidth]{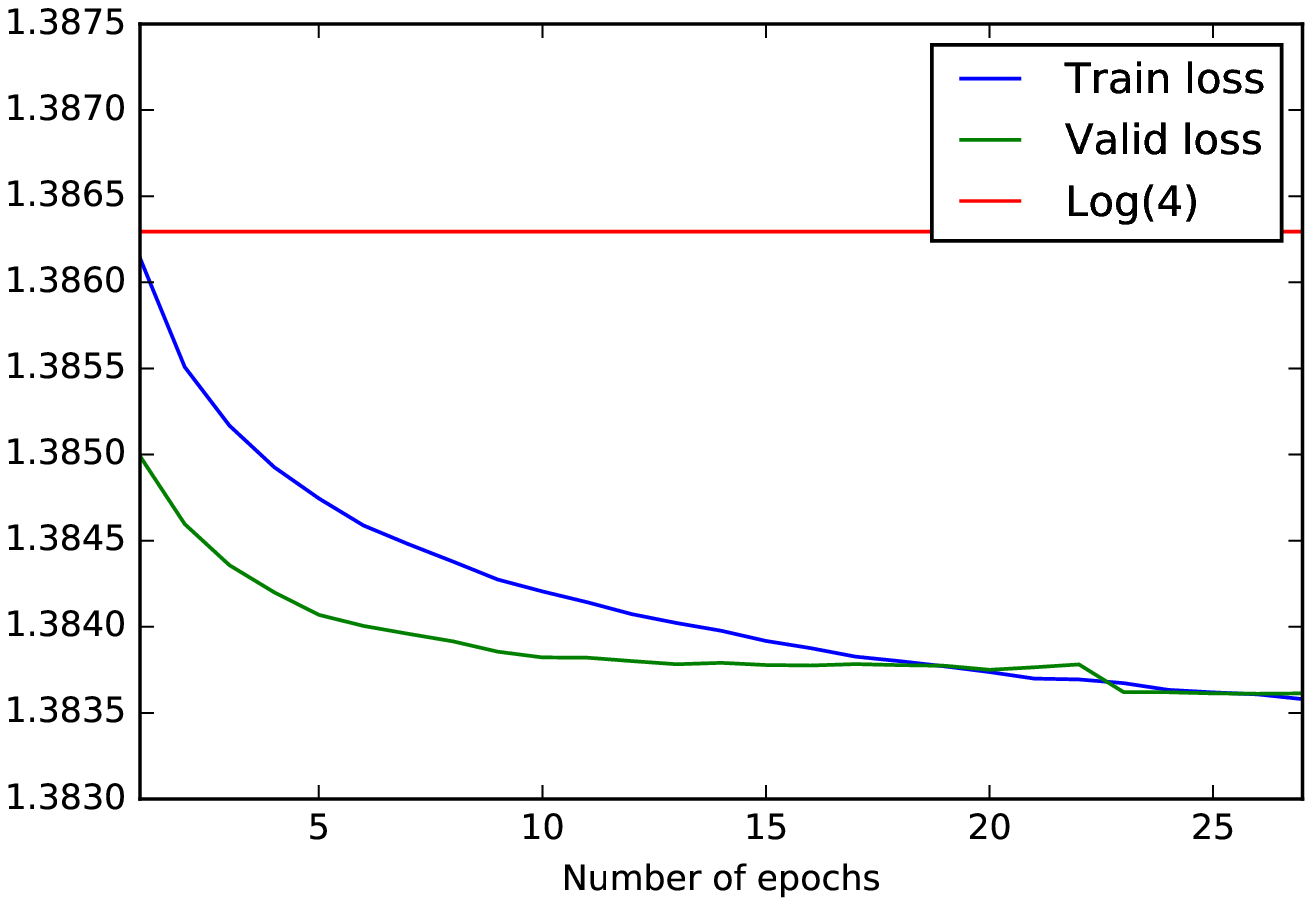} \\ }
\end{minipage}
\hfill
\begin{minipage}[h]{0.49\linewidth}
\center{\includegraphics[width=1.0\linewidth]{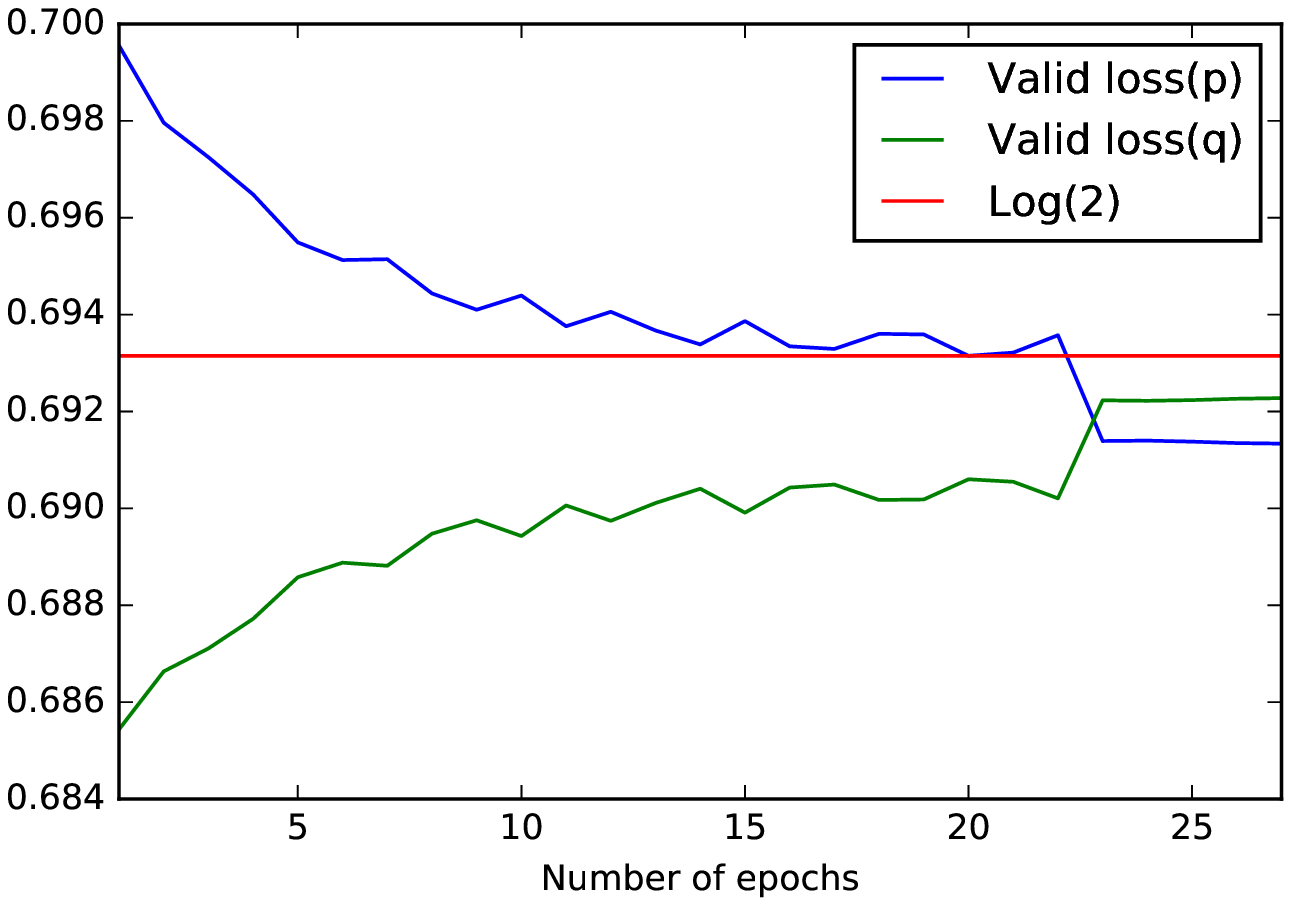} \\ }
\end{minipage}
\caption{Discriminator training curves (PTB). Marks ($p$) and ($q$) on the right plot mean the same as corresponding superscripts in (\ref{eq:disc_loss_detail}).}
\label{pic:Disc_train_curves}
\end{figure}

We started with current state-of-the-art language model based on the idea of direct output connection proposed by \cite{DOC_LM} which is essentially a generalization of the idea of mixture of softmaxes (\cite{MoS}). All parameters and hyperparameters were taken from the paper that introduces direct output connection architecture. Discriminator's architecture is the same as that of the language model. Training scheme is also the same -- we train discriminator with Stochastic Gradient Descent (SGD) and switch to Averaged Stochastic Gradient Descent (ASGD) when the quality ceases to improve.

Training curves on Figure \ref{pic:Disc_train_curves} show that discriminator training is very stable and fast -- it takes around $30$ epochs to converge. This is of no surprise since discriminator training is in fact training a language model with NCE loss which is known to be stable (see (\cite{NCE_LM})). On the figure there is a line corresponding to loss equal to $\log{4}$. This is the best loss discriminator can have in case the initial language model is optimal (i.e. $q_{0}\equiv p$). Since language model is not, discriminator can obtain lower level of loss which we actually observe.
Also, we noted that it is not so important whether we apply dropout and other regularization techniques used in the original paper ($L_{2}$-redularization on activations) or not -- the results differ insignificantly. Discriminator training dependence on learning rate is also mild.

We found out that the best performance on validation set was when we applied both dropout and all other regularization techniques and set learning rate to $1.0$ (with gradient clipping at $0.25$), so for further experiments with fine-tuning we chose this discriminator.

\begin{table}[t]
\caption{Quality after fine-tuning (PTB). ``Without discriminator'' means that for these estimators we took $r(w|c)\equiv 0.5$.}
\begin{center}
\begin{tabular}{|c|c|c|c|c|c|c|}
\hline
Fine-tuning Option &Val. loss &Val. ppl. &Val. rev. KL-div. &Test ppl.\\ 
\hline
Initial model &$3.99500$ &$54.31$ &$0.00024$ &$52.45$\\
\hline
Without dropout, with discriminator &$4.00248$ &$\textbf{53.90}$ &$0.01527$ &$\textbf{52.12}$\\
\hline
With dropout, with discriminator & \multicolumn{4}{c|}{both total loss and perplexity increase on validation set}\\
\hline
Without dropout, without discriminator &$4.01317$ &$54.16$ &$0.02126$ &$52.38$\\
\hline
With dropout, without discriminator & \multicolumn{4}{c|}{both total loss and perplexity increase on validation set}\\
\hline
\end{tabular}
\end{center}
\label{tab:ftd_loss_ptb}
\end{table}

The results of fine-tuning with the procedure proposed in section \ref{finetuner_general} are given in Table \ref{tab:ftd_loss_ptb}. It is possible to fine-tune the language model with or without dropout (in the first case the same dropout probabilities are chosen as those from training the initial language model). To show that discriminator is an important part of our fine-tuning algorithm, we made experiment with ``random'' discriminator (i.e. when in formula (\ref{eq:disc_expression}) we have $\frac{1}{2}$ for all words and contexts). Unlike discriminator training, we found out that learning rate has some impact on the performance of fine-tuned model. As for dropout, Table \ref{tab:ftd_loss_ptb} demonstrates that it definitely has negative effect on fine-tuning. Perhaps the reason is that reverse KL-divergence estimation in (\ref{eq:ftd_loss_eval}) is much more sensitive to the variance that is introduced into the whole scheme by dropout than cross-entropy estimation.

All models from Table \ref{tab:ftd_loss_ptb} were trained with learning rate $1.0$ (with gradient clipping at $0.25$) that was chosen because it minimized validation perplexity for the best of fine-tuning options mentioned in the table. It took only $15$ epochs to fine-tune the language model in the setting that resulted in state-of-the-art quality. This is much faster compared to fine-tuning with ASGD in (\cite{DOC_LM}) and (\cite{MoS}) where this process took hundreds of epochs. Moreover, a substantial decrease in perplexity is reached already after the first epoch of fine-tuning.

\subsection{WikiText-2}

\begin{table}[t]
\centering
\caption{Perplexity on PTB and WikiText-2.}
\begin{tabular}{|c||c|c|c|c|}
\hline
\multirow{2}{*}{Model} 
& \multicolumn{2}{c|}{\begin{tabular}[c]{@{}c@{}}Penn Treebank\end{tabular}} & \multicolumn{2}{c|}{\begin{tabular}[c]{@{}c@{}}WikiText-2\end{tabular}} \\ \cline{2-5}
& Valid & Test & Valid  & Test\\ 
\hline
Mixture of softmaxes, original paper & $56.54$ & $54.44$ & $63.88$ & $61.45$                                                                                \\ \hline
Mixture of softmaxes, results of our training & $56.77$ & $54.66$ & $64.35$ & $61.90$
\\ \hline
Mixture of softmaxes, our fine-tuning & $55.75$ & $\textbf{53.91}$ & $63.52$ & $\textbf{61.19}$
\\ \hline
Mixture of softmaxes, original paper + dyn. eval. & $48.33$ & $47.69$ & $42.41$ & $\textbf{40.68}$                                                                                
\\ \hline
Mixture of softmaxes, our fine-tuning + dyn. eval. & $48.25$ & $\textbf{47.68}$ & $42.83$ & $40.96$
\\ \hline\hline
Direct output connection, original paper & $54.12$ & $52.38$ & $60.29$ & $58.03$
\\ \hline
Direct output connection, results of our training & $54.31$ & $52.45$ & $--$ & $--$
\\ \hline
Direct output connection, our fine-tuning & $53.90$ & $\textbf{52.12}$ & $--$ & $--$
\\ \hline
\end{tabular}
\label{tab:SOTA}
\end{table}

The model utilizing mixture of softmaxes idea (\cite{MoS}) was used for the experiments on WikiText-2 (\cite{WT2}). We chose the same parameters and hyperparameters that were used in the paper. However, we have to note that we couldn't reproduce results from (\cite{MoS}): the results of our training were slightly worse and perplexity differed by $0.2-0.4$ on PTB and WikiText-2 (see Table~\ref{tab:SOTA}).

Discriminator training and language model fine-tuning took $25$ and $8$ epochs correspondingly. Learning rates for both were set to $1.0$ because they resulted in best performance on validation set. These were the only hyperparameters that we had to tune. Similarly to PTB, discriminator was trained with dropout (the same dropout rates as those used for training initial model). Fine-tuning was performed without dropout, since we've seen in section \ref{exp_ptb} that this is a necessary condition for our algorithm to work well.

Perplexity of the fine-tuned models (both for PTB and WikiText-2) are shown in Table \ref{tab:SOTA}. Dynamic evaluation parameters for models based on mixture of softmaxes were tuned using grid search on validation set. It's worth noting that perplexity gain in dynamic evaluation mode (\cite{DynEval}) is very mild if any.

\subsection{Large-scale experiment}

Large-scale dataset we trained our model on consisted of $4$Gb of English texts gathered from the Internet (blogs, news articles, forums, etc.) Vocabulary consisted of top $100$k words from training corpus. Validation and test sets' sizes were around $100$Mb. We chose a very simple recurrent model -- a single-layer LSTM with $500$ hidden units. We used differentiated softmax (\cite{DifSoftmax}) with output embeddings dimension varying from $150$ for high-frequency words to $16$ for low-frequency words. Neither dropout nor any other regularization technique was applied during training with SGD since data size was enough for good generalization.

The initial model perplexity $129.08$ went down to $127.46$ as a result of fine-tuning with discriminator. Discriminator training and model fine-tuning took $60$ and $30$ epochs correspondingly. All data was split into $20$Mb parts and by epoch here we mean full pass through such part, i.e. the model was fine-tuned by one pass through approximately $600$Mb of data.

\section{Conclusion}
We have presented a novel optimization trick that makes it possible to fine-tune language models at any scale. The key idea is to modify cross-entropy loss by adding reverse KL-divergence that is estimated using discriminator that has to be trained first. It enables model to decide whether probability of current word is underestimated or overestimated, and to choose the direction of gradient step depending on that. As a result, rare words probabilities which are usually underestimated increase without significant damage to frequent words probabilities. This approach allowed us to reach state-of-the-art quality in language modeling task on a popular benchmark Penn Treebank. Apart from being effective, our approach is also fast and easy to use since it requires almost no hyperparameter tuning.

\bibliography{finetune_paper}
\bibliographystyle{common_style}

\appendix
\section{Details of experiments on synthetic data}\label{artdataset}

Vocabulary consisted of $1000$ words and an additional token that every sentence started with. All sentences consisted of $10$ words excluding starting token. Validation and test sets consisted of $10$k sentences while training set size was $80$k sentences. These sets were sampled from trigram distribution.

For natural numbers $i, j, k$ not exceeding $1000$ probability of each trigram $(i, j, k)$ was proportional to $\alpha(j)\cdot\beta(i, j)\cdot\beta(k, j)$:
\begin{equation*}
\alpha(j)=
\begin{cases}
\frac{1}{50}\cdot\frac{\log{(j+1)}}{\log{51}} \text{ \quad for $j\leq 50$}\\
\frac{1}{j} \text{ \qquad \qquad \qquad otherwise}
\end{cases}
\end{equation*}
$$\beta(i,j)=(i\cdot(|i-j+\xi|+1))^{-\gamma}$$
where $\xi$ is a random variable with uniform distribution on the set $\{0, 1, 2, 3, 4, 5\}$ sampled independently for each triple $(i, j, k)$ and $\gamma=0.75$.

Such distribution leads to separation into frequent and rare words. Top $50$ most frequent words constitute $70.5\%$ of training corpus, the next $100$ words constitute $14.7\%$ (the chosen word belongs to this group), and the remaining $850$ words ratio is $14.8\%$ of training corpus.

Two-layer LSTM network with $256$ hidden untis and the same embedding dimensionality was used both for language model and discriminator. Both language model and discriminator were trained with SGD with batches containing $1024$ sentences and learning rates equal to $1.0$ (with gradient clipping at $1.0$). Learning rates were multiplied by $0.1$ each time validation perplexity showed no improvement. Neither dropout nor any other regularization technique was applied.

\section{Proposition proof}\label{appendix}
Proof of formula (\ref{eq:ce_part_grad}) (recall that $q_{\theta}=\hat{p}-\varepsilon$):
$$\frac{\partial}{\partial q_{\theta}}\left(-\log{q_{\theta}}+\frac{q_{\theta}}{\hat{p}}\log{\frac{q_{\theta}}{\hat{p}}}\right)=\left(-\frac{1}{q_{\theta}}+\frac{1}{\hat{p}}\left(1+\log{\frac{q_{\theta}}{\hat{p}}}\right)\right)=-\left(\frac{1}{\hat{p}-\varepsilon}-\frac{1}{\hat{p}}\right)+\frac{1}{\hat{p}}\log{\left(1-\frac{\varepsilon}{\hat{p}}\right)}$$
$$=-\frac{\varepsilon}{\hat{p}^{2}}\left(1-\frac{\varepsilon}{\hat{p}}\right)^{-1}+\frac{1}{\hat{p}}\left(-\frac{\varepsilon}{\hat{p}}-\frac{\varepsilon^{2}}{2\hat{p}^{2}}+o(\varepsilon^{2})\right)=-\frac{\varepsilon}{\hat{p}^{2}}\left(1+\frac{\varepsilon}{\hat{p}}+o(\varepsilon)\right)-\frac{\varepsilon}{\hat{p}^{2}}-\frac{\varepsilon^{2}}{2\hat{p}^{3}}+o(\varepsilon^{2})$$
$$=\left(-\frac{2\varepsilon}{\hat{p}^{2}}-\frac{3}{2}\frac{\varepsilon^{2}}{\hat{p}^{3}}+o(\varepsilon^{2})\right)$$
Proof of formula (\ref{eq:klrev_part_grad}) (recall that $q_{\theta}=\hat{p}+\varepsilon$):
$$\frac{\partial}{\partial q_{\theta}}\left(-\log{q_{\theta}}+\frac{q_{\theta}}{\hat{p}}\log{\frac{q_{\theta}}{\hat{p}}}\right)=\left(-\frac{1}{q_{\theta}}+\frac{1}{\hat{p}}\left(1+\log{\frac{q_{\theta}}{\hat{p}}}\right)\right)=\left(\frac{1}{\hat{p}}-\frac{1}{\hat{p}+\varepsilon}\right)+\frac{1}{\hat{p}}\log{\left(1+\frac{\varepsilon}{\hat{p}}\right)}$$
$$=\frac{\varepsilon}{\hat{p}^{2}}\left(1+\frac{\varepsilon}{\hat{p}}\right)^{-1}+\frac{1}{\hat{p}}\left(\frac{\varepsilon}{\hat{p}}-\frac{\varepsilon^{2}}{2\hat{p}^{2}}+o(\varepsilon^{2})\right)=\frac{\varepsilon}{\hat{p}^{2}}\left(1-\frac{\varepsilon}{\hat{p}}+o(\varepsilon)\right)+\frac{\varepsilon}{\hat{p}^{2}}-\frac{\varepsilon^{2}}{2\hat{p}^{3}}+o(\varepsilon^{2})$$
$$=\left(\frac{2\varepsilon}{\hat{p}^{2}}-\frac{3}{2}\frac{\varepsilon^{2}}{\hat{p}^{3}}+o(\varepsilon^{2})\right)$$

\end{document}